\documentclass[11pt]{article}
\usepackage[margin=1in]{geometry}
\usepackage{amsmath}
\usepackage{amssymb}
\usepackage{mathtools}
\usepackage{amsthm}
\usepackage{dsfont}
\usepackage{color}
\usepackage{graphicx}
\usepackage[numbers]{natbib}
\usepackage{hyperref}

\newtheorem{theorem}{Theorem}

\newcommand{\E}{{\mathbb{E}}}
\newcommand{\R}{{\mathbb{R}}}
\newcommand{\cH}{{\mathcal{H}}}
\newcommand{\cHN}{{\mathcal{H}_N}}
\newcommand{\rkhs}{{\mathcal{H_\infty}}}

\newcommand\ind[1]{\mathds{1}_{\{#1\}}}

\newcommand\optnN{h_{n,N}}
\newcommand\optninf{h_{n,\infty}}

\title{Reconciling modern machine learning practice \\ and the bias-variance trade-off}

\usepackage[affil-sl]{authblk}
\author[a]{Mikhail Belkin}
\author[b]{Daniel Hsu}
\author[a]{Siyuan Ma}
\author[a]{Soumik Mandal}
\affil[a]{The Ohio State University, Columbus, OH}
\affil[b]{Columbia University, New York, NY}

\begin{document} 
{\def\thefootnote{}
\footnotetext{E-mail:
\texttt{mbelkin@cse.ohio-state.edu},
\texttt{djhsu@cs.columbia.edu},
\texttt{masi@cse.ohio-state.edu},
\texttt{mandal.32@osu.edu}}}

\maketitle

\begin{abstract}
Breakthroughs in machine learning are rapidly changing science and society, yet our fundamental understanding of this technology has lagged far behind.
Indeed, one of the central tenets of the field, the bias-variance trade-off, appears to be at odds with the observed behavior of methods used in the modern machine learning practice.
The bias-variance trade-off implies that a model should balance under-fitting and over-fitting: rich enough to express underlying structure in data, simple enough to avoid fitting spurious patterns.
However, in the modern practice, very rich models such as neural networks are trained to exactly fit (i.e., interpolate) the data.
Classically, such models would be considered over-fit, and yet they often obtain high accuracy on test data.
This apparent contradiction has raised questions about the mathematical foundations of machine learning and their relevance to practitioners.

In this paper, we reconcile the classical understanding and the modern practice within a unified performance curve.
This   ``double descent'' curve  subsumes the textbook U-shaped bias-variance trade-off curve by showing how increasing model capacity beyond the point of interpolation results in improved performance. We provide evidence for the existence and ubiquity of double descent for a wide spectrum of models and datasets, and we posit a mechanism for its emergence. This connection between the performance and the structure of machine learning models delineates the limits of classical analyses, and has implications for both the theory and practice of machine learning.
\end{abstract}

\newpage

\section{Introduction}\label{sec:intro}

Machine learning has become key to important applications in science, technology and commerce.
The focus of machine learning is on the problem of prediction:
given a sample of training examples $(x_1,y_1),\dotsc,(x_n,y_n)$ from $\R^d \times \R$, we learn a predictor $h_n \colon \R^d \to \R$ that is used to predict the label $y$ of a new point $x$, unseen in training.
 
The predictor $h_n$ is commonly chosen from some function class $\cH$, such as  neural networks with a certain architecture, using \emph{empirical risk minimization (ERM)} and its variants.
In ERM, the predictor is taken to be a function $h \in \cH$ that minimizes the \emph{empirical (or training) risk} $\frac1n \sum_{i=1}^n \ell(h(x_i), y_i)$, where $\ell$ is a loss function, such as the squared loss $\ell(y',y) = (y' - y)^2$ for regression or zero-one loss $\ell(y',y) = \ind{y' \neq y}$ for classification.

The goal of machine learning is to find $h_n$ that performs well on new data, unseen in training. 
To study performance on new data  (known as {\it generalization}) we typically assume the training examples are sampled randomly from a probability distribution $P$ over $\R^d \times \R$, and evaluate $h_n$ on a new test example $(x,y)$ drawn independently from $P$. The challenge stems from the mismatch between the goals of minimizing the empirical risk (the explicit goal of ERM algorithms, optimization) and minimizing the \emph{true (or test) risk} $\E_{(x,y)\sim P}[\ell(h(x),y)]$ (the goal of machine learning).

Conventional wisdom in machine learning suggests controlling the capacity of the function class $\cH$ based on the {\em bias-variance trade-off} by balancing \emph{under-fitting} and \emph{over-fitting} (cf.,~\cite{geman1992neural,friedman2001elements}): 
\begin{enumerate}
  \item
    If $\cH$ is too small, all predictors in $\cH$ may \emph{under-fit} the training data (i.e., have large empirical risk) and hence predict poorly on new data.

  \item
    If $\cH$ is too large, the empirical risk minimizer may \emph{over-fit} spurious patterns in the training data resulting in poor accuracy on new examples (small empirical risk but large true risk).
\end{enumerate}
The classical thinking is concerned with finding the ``sweet spot'' between under-fitting and over-fitting.
The control of the function class capacity may be explicit, via the choice of $\cH$ (e.g., picking the neural network architecture), or it may be implicit, using regularization (e.g., early stopping).
When a suitable balance is achieved, the performance of $h_n$ on the training data is said to \emph{generalize} to the population $P$.
This is summarized in the classical U-shaped risk curve, shown in Figure~\ref{fig:double-descent}(a) that has been widely used to guide  model selection and is even thought to describe aspects of human decision making~\cite{gigerenzer2009homo}.
The textbook corollary of this curve is that ``a model with zero training error is overfit to the training data and will typically generalize poorly''~\cite[page 221]{friedman2001elements}, a view still widely accepted.

Yet, practitioners routinely use modern machine learning methods, such as large neural networks and other non-linear predictors that have very low or  zero training risk. In spite of the high function class capacity and near-perfect fit to training data, these predictors often  give very accurate predictions on new data.
Indeed, this behavior has guided a best practice in deep learning for choosing neural network architectures, specifically that the network  should be  large enough to permit effortless zero loss training (called {\it interpolation}) of the training data~\cite{russ17}.
Moreover, in direct challenge to the bias-variance trade-off philosophy,
recent empirical evidence indicates that neural networks and kernel machines trained to interpolate the training data obtain near-optimal test results even when the training data are corrupted with high levels of noise~\cite{zhang2016understanding,belkin2018understand}.

\begin{figure}
  \centering
  \begin{tabular}{cc}
  \includegraphics[height=0.15\textheight]{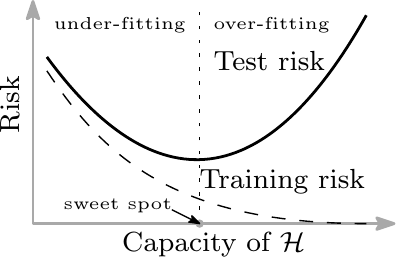} &
  \includegraphics[height=0.15\textheight]{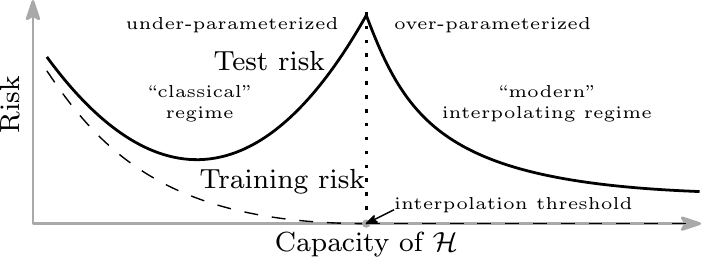} \\
  {\bf (a)} & {\bf (b)}
  \end{tabular}
  \caption{{\bf Curves for training risk (dashed line) and test risk (solid line).}
    ({\bf a}) The classical \emph{U-shaped risk curve} arising from the bias-variance trade-off.
    ({\bf b}) The \emph{double descent risk curve}, which incorporates the U-shaped risk curve (i.e., the ``classical'' regime) together with the observed behavior from using high capacity function classes (i.e., the ``modern'' interpolating regime), separated by the interpolation threshold.
    The predictors to the right of the interpolation threshold have zero training risk.}
  \label{fig:double-descent}
\end{figure}

The main finding of this work is a pattern for how performance on unseen data depends on model capacity and the mechanism underlying its emergence. This dependence, empirically witnessed with important model classes including neural networks and a range of datasets, is
summarized in the ``double descent'' risk curve shown in Figure~\ref{fig:double-descent}(b). The curve subsumes the  classical U-shaped risk curve from Figure~\ref{fig:double-descent}(a) by extending it beyond the point of interpolation.

When function class capacity is below the ``interpolation threshold'', learned predictors exhibit the classical U-shaped curve from Figure~\ref{fig:double-descent}(a).
(In this paper, function class capacity is identified with the number of parameters needed to specify a function within the class.)
The bottom of the U is achieved at the sweet spot which balances the fit to the training data and the susceptibility to over-fitting:
to the left of the sweet spot, predictors are under-fit, and immediately to the right, predictors are over-fit.
When we increase the function class capacity high enough (e.g., by increasing the number of features or the size of the neural network architecture), the learned predictors achieve (near) perfect fits to the training data---i.e., interpolation.
Although the learned predictors obtained at the interpolation threshold typically have high risk, we show that increasing the function class capacity beyond this point leads to decreasing risk, typically going below the risk achieved at the sweet spot in the ``classical'' regime.

All of the learned predictors  to the right of the interpolation threshold fit the training data perfectly and have zero empirical risk.
So why should some---in particular, those from richer functions classes---have lower test risk than others?
The answer is that the capacity of the function class does not necessarily reflect how well the predictor matches the \emph{inductive bias} appropriate for the problem at hand.
For the learning problems we consider (a range of real-world datasets as well as synthetic data), the inductive bias that seems appropriate  is the regularity or smoothness of a function as measured by a certain function space norm.
Choosing the smoothest function that perfectly fits observed data is a form of Occam's razor: the simplest explanation compatible with the observations should be preferred (cf.~\cite{Vapnik,blumer1987occam}).
By considering larger function classes, which contain more candidate predictors compatible with the data, we are able to find interpolating functions that have smaller norm and are thus ``simpler''. 
Thus increasing function class capacity improves performance of classifiers.

Related ideas have been considered in the context of margins theory~\cite{Vapnik,bartlett1998sample,schapire1998}, where a larger function class $\cH$ may permit the discovery of a classifier with a larger margin.
While the margins theory can be used to study classification, it does not apply to regression, and also does not predict the second descent beyond the interpolation threshold.
Recently, there has been an emerging recognition that certain interpolating predictors (not based on ERM) can indeed be provably statistically optimal or near-optimal~\cite{belkin2018overfitting,belkin2018does}, which is compatible with our empirical observations in the interpolating regime.

In the remainder of this article, we discuss empirical evidence for the double descent curve, the mechanism for its emergence and conclude with some final observations and parting thoughts.

\section{Neural networks}
\label{sec:RFF+NN}

In this section, we discuss the double descent risk curve in the context of neural networks.

\paragraph{Random Fourier features.}

We first consider a popular class of non-linear parametric models called \emph{Random Fourier Features} (\emph{RFF})~\cite{rahimi2008random}, which can be viewed as a  class of two-layer neural networks with fixed weights in the first layer.
The RFF model family $\cHN$ with $N$ (complex-valued) parameters consists of functions $h \colon \R^d \to \mathbb{C}$ of the form
\[
  h(x) = \sum_{k=1}^N a_k \phi(x; v_k)
  \quad\text{where}\quad
  \phi(x; v) := e^{\sqrt{-1} \langle v,x \rangle} ,
\]
and the vectors $v_1,\dotsc,v_N$ are sampled independently from the standard normal distribution in $\R^d$.
(We consider $\cHN$ as a class of real-valued functions with $2N$ real-valued parameters by taking real and imaginary parts separately.)
Note that $\cHN$ is a randomized function class, but as $N \to \infty$, the function class becomes a closer and closer approximation to the Reproducing Kernel Hilbert Space (RKHS) corresponding to the Gaussian kernel, denoted by $\rkhs$.
While it is possible to directly use $\rkhs$ (e.g., as is done with kernel machines~\cite{boser1992training}), the random classes $\cHN$ are computationally attractive to use when the sample size $n$ is large but the number of parameters $N$ is small compared to $n$.

Our learning procedure using $\cHN$ is as follows.
Given data $(x_1,y_1),\dotsc,(x_n,y_n)$ from $\R^d \times \R$, we find the predictor $\optnN \in \cHN$ via ERM with squared loss.
That is, we minimize the empirical risk objective $\tfrac1n \sum_{i=1}^n (h(x_i) - y_i)^2$ over all functions $h \in \cHN$.
When the minimizer is not unique (as is always the case when $N>n$), we choose the minimizer whose coefficients $(a_1,\dotsc,a_N)$ have the minimum $\ell_2$ norm.
This choice of norm is intended as an approximation to the RKHS norm $\|h\|_{\rkhs}$, which is generally difficult to compute for arbitrary functions in $\cHN$.
For problems with multiple outputs (e.g., multi-class classification), we use functions with vector-valued outputs and sum of the squared losses for each output.

\begin{figure}
  \centering
  \includegraphics[width=\textwidth]{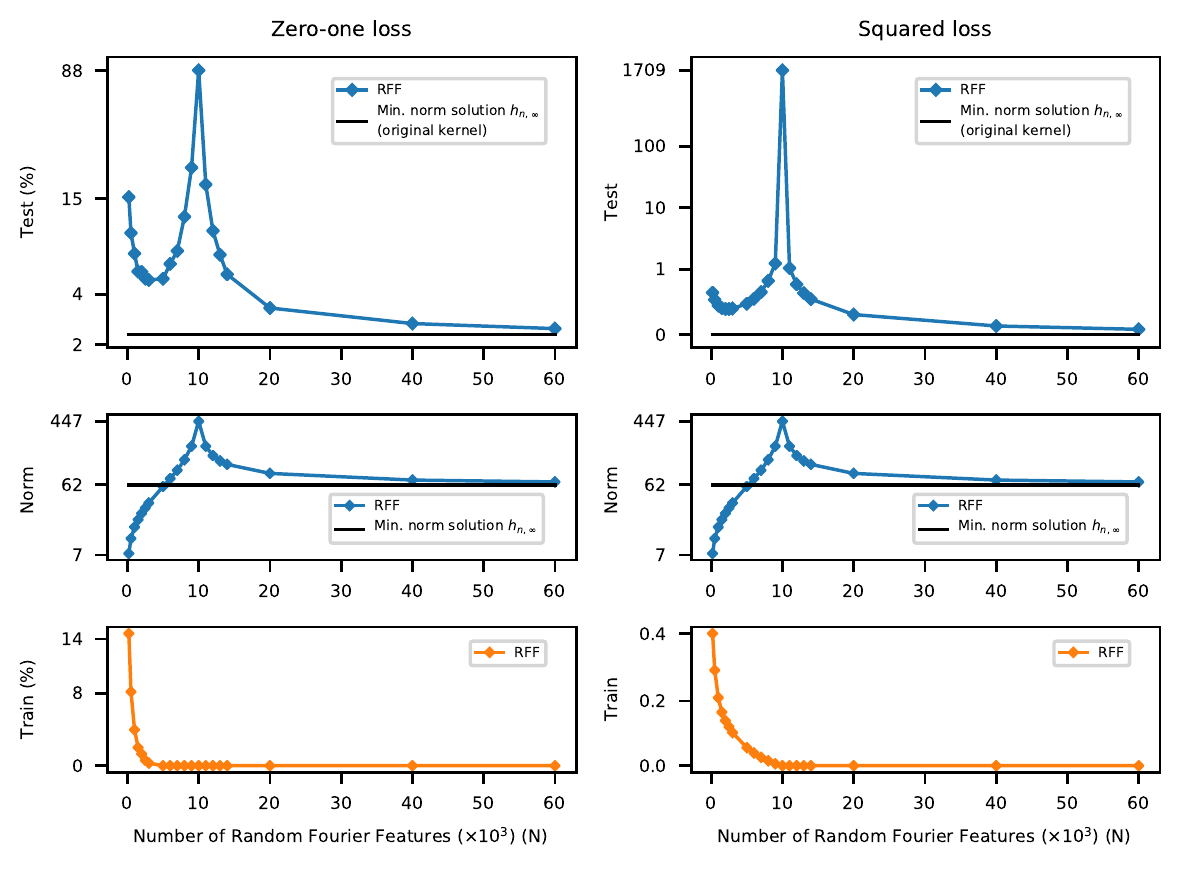}
  \caption{ {\bf Double descent risk curve for RFF model on MNIST.} Test risks (log scale), coefficient $\ell_2$ norms (log scale), and training risks of the RFF model predictors $\optnN$ learned on a subset of MNIST ($n=10^4$, $10$ classes). The interpolation threshold is achieved at $N=10^4$.}
  \label{fig:RFF_MNIST}
\end{figure}

In Figure~\ref{fig:RFF_MNIST}, we show the test risk of the predictors learned using $\cHN$ on a subset of the popular data set of handwritten digits called MNIST.
The same figure also shows the $\ell_2$ norm of the function coefficients, as well as the training risk.
We see that for small values of $N$, the test risk shows the classical U-shaped curve consistent with the bias-variance trade-off, with a peak occurring at the interpolation threshold $N=n$.
Some statistical analyses of RFF suggest choosing $N \propto \sqrt{n}\log n$ to obtain good test risk guarantees~\cite{rudi2017generalization}.

The interpolation regime connected with modern practice is shown to the right of the interpolation threshold, with $N \geq n$.
The model class that achieves interpolation with fewest parameters ($N=n$ random features) yields the least accurate predictor.
(In fact, it has no predictive ability for classification.)
But as the number of features increases beyond $n$, the accuracy improves dramatically, exceeding that of the predictor corresponding to the bottom of the U-shaped curve.
The plot also shows that the predictor $\optninf$ obtained from $\rkhs$ (the kernel machine) out-performs the predictors from $\cHN$ for any finite $N$.

What structural mechanisms account for the double descent shape?
When the number of features is much smaller then the sample size, $N \ll n$, classical statistical arguments imply that the training risk is close to the test risk.
Thus, for small $N$, adding more features yields improvements in both the training and test risks.
However, as the number of features approaches $n$ (the interpolation threshold), features not present or only weakly present in the data are forced to fit the training data nearly perfectly.
This results in classical over-fitting as predicted by the bias-variance trade-off and prominently manifested at the  peak of the curve, where the fit becomes exact.

To the right of the interpolation threshold, all function classes are rich enough to achieve zero training risk.
For the classes $\cHN$ that we consider, there is no guarantee that the most regular, smallest norm predictor consistent with training data (namely $\optninf$, which is in $\rkhs$) is contained in the class $\cHN$ for any finite $N$.
But increasing $N$ allows us to construct progressively better approximations to that smallest norm function.
Thus we expect to have learned predictors with largest norm at the interpolation threshold and for the norm of $\optnN$ to decrease monotonically as $N$ increases thus explaining the second descent segment of the curve.  
This is what we observe in Figure~\ref{fig:RFF_MNIST}, and indeed $\optninf$ has better accuracy than all $\optnN$ for any finite $N$. 
Favoring small norm interpolating predictors turns out to be a powerful inductive bias on MNIST and  other real and synthetic data sets~\cite{belkin2018understand}.
For noiseless data, we make this claim mathematically precise in Appendix~\ref{app:approx}.

Additional empirical evidence for the same double descent behavior using other data sets is presented in Appendix~\ref{app:RFF}.
For instance, we demonstrate double descent for rectified linear unit (ReLU) random feature models, a class of ReLU neural networks with a setting similar to that of RFF. The inductive bias corresponding to the larger number of features can be readily observed in a one-dimensional example in Figure~\ref{fig:rrf-plot}.
Although the fitted function is  non-smooth (piecewise linear) for any number of Random ReLU features,  it appears smoother---with smaller norm---as the number of features is increased.

Finally, in Appendix~\ref{app:RFF1D}, we also describe a simple synthetic model, which can be regarded as a one-dimensional version of the RFF model, where we observe the same double descent behavior.

\begin{figure}
  \centering
  \includegraphics[width=.70\textwidth]{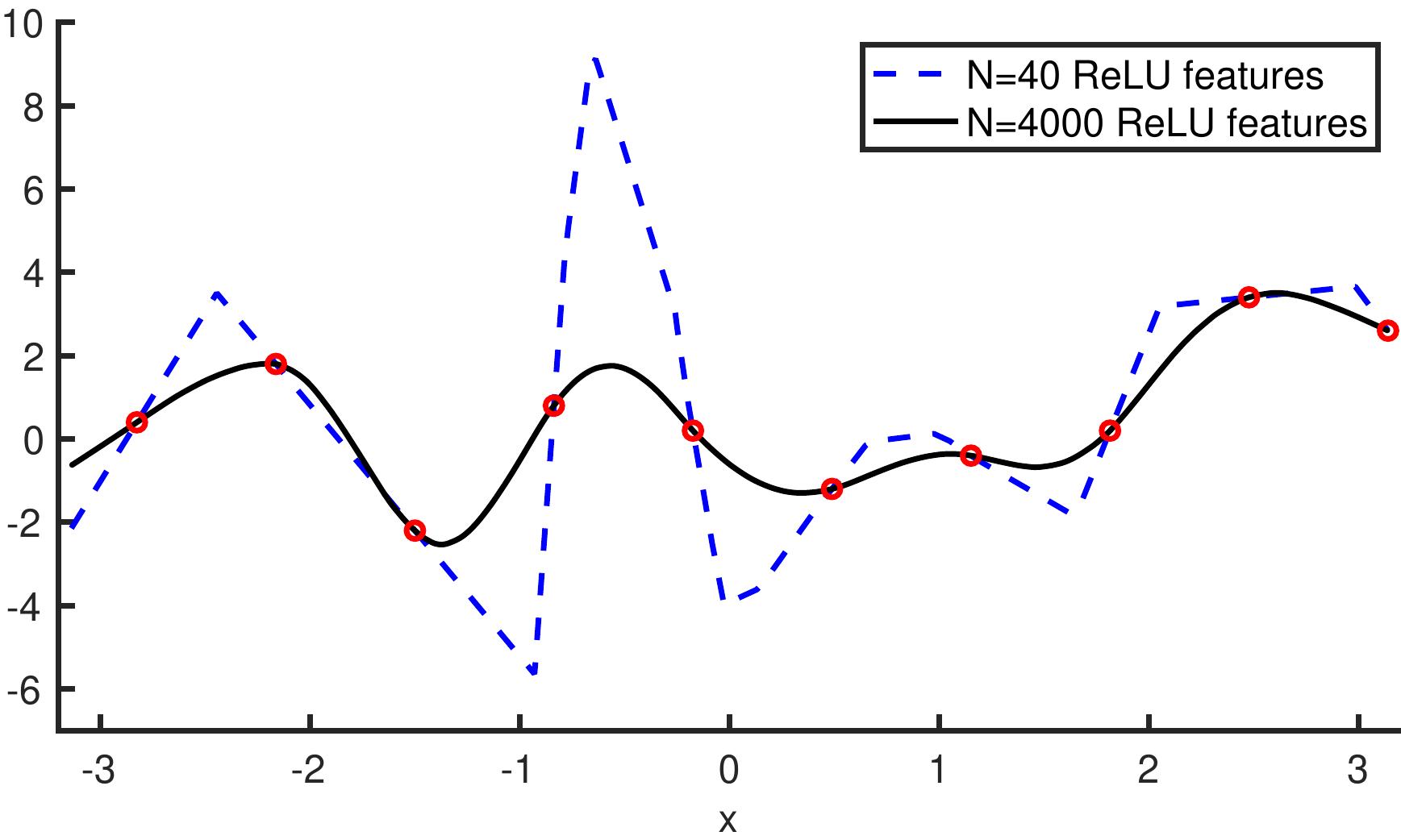}
  \caption{Plot of two univariate functions fitted to $10$ data points using Random ReLU features $\phi(x; (v_1,v_2)) := \max( v_1x + v_2, 0)$.
    The data points are shown in red circles.
    The fitted function with $N=40$ Random ReLU features is the blue dashed line; the coefficient vector's norm (scaled by $\sqrt{N}$) is $\approx 695$.
    The fitted function with $N=4000$ Random ReLU features is the black solid line; the coefficient vector's norm is $\approx 159$.}
  \label{fig:rrf-plot}
\end{figure}

\paragraph{Neural networks and backpropagation.}

In general multilayer neural networks (beyond RFF or ReLU random feature models), a learning algorithm will tune all of the weights to fit the training data, typically using versions of stochastic gradient descent (SGD), with backpropagation to compute partial derivatives.
This flexibility increases the representational power of neural networks, but also makes ERM generally more difficult to implement.
Nevertheless, as shown in Figure~\ref{fig:FuCoNN_MNIST_REUSEWEIGHT}, we observe that increasing the number of parameters in fully connected two-layer neural networks leads to a risk curve qualitatively similar to that observed with RFF models.
That the test risk improves beyond the interpolation threshold is compatible with the conjectured ``small norm'' inductive biases of the common training algorithms for neural networks~\cite{gunasekar2017implicit,pmlr-v75-li18a}.
We note that this transition from under- to over-parameterized regimes for neural networks was also previously observed by~\cite{bos1997dynamics,advani2017high,neal2018modern,spigler2018jamming}.
In particular, \cite{spigler2018jamming} draws a connection to the physical phenomenon of ``jamming'' in particle systems.

The computational complexity of ERM with neural networks makes the double descent risk curve difficult to observe.
Indeed, in the classical under-parametrized regime ($N \ll n$), the non-convexity of the ERM optimization problem causes the behavior of local search-based heuristics, like SGD, to be highly sensitive to their initialization.
Thus, if only suboptimal solutions are found for the ERM optimization problems, increasing the size of a neural network architecture may not always lead to a corresponding decrease in the training risk.
This suboptimal behavior can lead to high variability in both the training and test risks that masks the double descent curve.

\begin{figure}\centering
\includegraphics[width=0.7\textwidth]{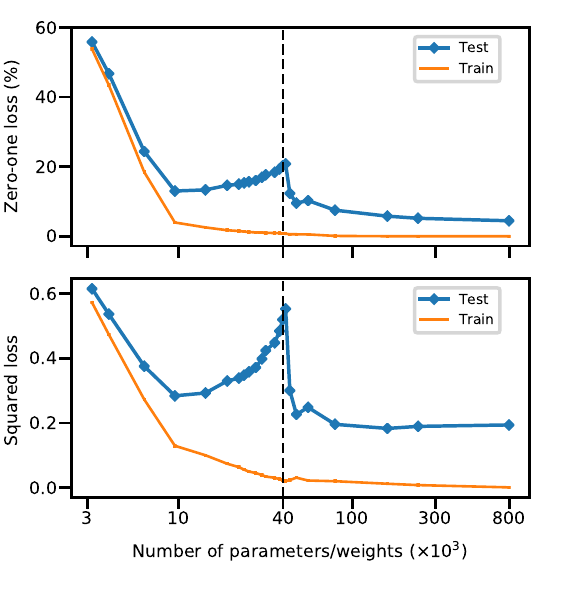}\caption{{\bf Double descent risk curve for fully connected neural network on MNIST.} 
  Training and test risks of network with a single layer of $H$ hidden units, learned on a subset of MNIST ($n=4\cdot10^3$, $d=784$, $K = 10$ classes).
  The number of parameters is $(d+1) \cdot H + (H+1) \cdot K$. 
  The interpolation threshold (black dotted line) is observed at $n \cdot K$. 
}\label{fig:FuCoNN_MNIST_REUSEWEIGHT}\end{figure}

It is common to use neural networks with extremely large number of parameters~\cite{canziani2016analysis}.
But to achieve interpolation for a single output (regression or two class classification) one expects to need at least as many parameters as there are data points.
Moreover, if the prediction problem has more than one output (as in multi-class classification), then the number of parameters needed should be multiplied by the number of outputs. This is indeed the case empirically for neural networks shown in Figure~\ref{fig:FuCoNN_MNIST_REUSEWEIGHT}.
Thus, for instance, data sets as large as ImageNet~\cite{ILSVRC15}, which has ${\sim}10^6$ examples and ${\sim}10^3$ classes, may require networks with ${\sim}10^9$ parameters to achieve interpolation; this is larger than many neural network models for ImageNet~\cite{canziani2016analysis}.
In such cases, the classical regime of the U-shaped risk curve is more appropriate to understand generalization.
For smaller data sets, these large neural networks  would be firmly in the over-parametrized regime, and simply training to obtain zero training risk often results in good test performance~\cite{zhang2016understanding}.

Additional results with neural networks are given in Appendix~\ref{app:FNN}.

\section{Decision trees and ensemble methods}

Does the double descent risk curve manifest with other prediction methods besides neural networks?
We give empirical evidence that the families of functions explored by boosting with decision trees and Random Forests also show similar generalization behavior as neural nets, both before and after the interpolation threshold.

AdaBoost and Random Forests have recently been investigated in the interpolation regime by \cite{wyner2017explaining} for classification.
In particular, they give empirical evidence that, when AdaBoost and Random Forests are used with maximally large (interpolating) decision trees, the flexibility of the fitting methods yield interpolating predictors that are more robust to noise in the training data than the predictors produced by rigid, non-interpolating methods (e.g., AdaBoost or Random Forests with shallow trees).
This in turn is said to yield better generalization.
The averaging of the (near) interpolating trees ensures that the resulting function is substantially smoother than any individual tree, which aligns with an inductive bias that is compatible with many real world problems.

\begin{figure}
  \centering
    \includegraphics[width=0.7\textwidth]{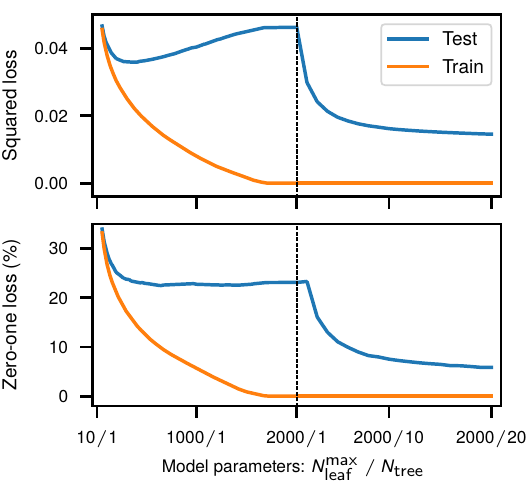}
  \caption{{\bf Double descent risk curve for random forests on MNIST.} The double descent risk curve is observed for random forests with increasing model complexity trained on a subset of MNIST ($n=10^4, 10$ classes). Its complexity is controlled by the number of trees $N_\mathsf{tree}$ and the maximum number of leaves allowed for each tree $N_\mathsf{leaf}^\mathsf{max}$.
  }
  \label{fig:forest-mnist-main}
\end{figure}

We can understand these flexible fitting methods in the context of the double descent risk curve.
Observe that the size of a decision tree (controlled by the number of leaves) is a natural way to parametrize the function class capacity: trees with only two leaves correspond to two-piecewise constant functions with axis-aligned boundary, while trees with $n$ leaves can interpolate $n$ training examples.
It is a classical observation that the U-shaped bias-variance trade-off curve manifests in many problems when the class capacity is considered this way~\cite{friedman2001elements}.
(The interpolation threshold may be reached with fewer than $n$ leaves in many cases, but $n$ is clearly an upper bound.)
To further enlarge the function class, we consider ensembles (averages) of several interpolating trees.\footnote{These trees are trained in the way proposed in Random Forest except without bootstrap re-sampling. This is similar to the PERT method of \cite{cutler2001pert}.}
So, beyond the interpolation threshold, we use the number of such trees to index the class capacity.
When we view the risk curve as a function of class capacity defined in this hybrid fashion, we see the double descent curve appear just as with neural networks; see Figure~\ref{fig:forest-mnist-main} and Appendix~\ref{app:forests}.
We observe a similar phenomenon using $L_2$-boosting~\cite{friedman2001greedy,buhlmann2003boosting}, another popular ensemble method; the results are reported in Appendix~\ref{app:boosting}.

\section{Concluding thoughts}
\label{sec:conclusion}

The double descent risk curve introduced in this paper reconciles the U-shaped curve predicted by the bias-variance trade-off and the observed behavior of rich models used in modern machine learning practice. The posited mechanism that underlies its emergence is based on common inductive biases, and hence can explain its appearance (and, we argue, ubiquity) in machine learning applications.

We conclude with some final remarks.
\paragraph{Historical absence.}
The double descent behavior may have been historically overlooked on account of several cultural and practical barriers.
Observing the double descent curve requires a parametric family of spaces with functions of arbitrary complexity.
The linear settings studied extensively in classical statistics usually assume a small, fixed set of features and hence fixed fitting capacity. 
Richer families of function classes are typically used in the context of non-parametric statistics, where smoothing and regularization are almost always employed~\cite{wasserman2006all}.
Regularization, of all forms, can both prevent interpolation and change the effective capacity of the function class, thus attenuating or masking the interpolation peak.

The RFF model is a  popular and flexible parametric family.
However, these models were originally proposed as computationally favorable alternative to kernel machines.
This computational advantage over traditional kernel methods holds only for $N \ll n$, and hence models at or beyond the interpolation threshold are typically not considered.

The situation with general multilayer neural networks, is slightly different and more involved.
Due to the non-convexity of the ERM optimization problem, solutions in the classical under-parametrized regime are highly sensitive to initialization.
Moreover, as we have seen, the peak at the interpolation threshold is observed within a narrow range of parameters. Sampling of the parameter space that misses that range may lead to the misleading impression that increasing the size of the network simply improves performance.
Finally, in practice, training of neural networks is typically stopped as soon as (an estimate of) the test risk fails to improve. This {\em early stopping} has a strong regularizing effect that, as discussed above, makes it difficult to observe the interpolation peak. 

\paragraph{Inductive bias.}
In this paper, we have dealt with several types of methods for choosing interpolating solutions.
For Random Fourier and Random ReLU features, solutions are constructed explicitly by minimum norm linear regression in the feature space.
As the number of features tends to infinity they approach the minimum functional norm solution in the Reproducing Kernel Hilbert Space, a solution which maximizes functional smoothness subject to the interpolation constraints. 
For neural networks, the inductive bias owes to the specific training procedure used, which is typically SGD.
When all but the final layer of the network are fixed (as in RFF models), SGD initialized at zero also converges to the minimum norm solution.
While the behavior of SGD for more general neural networks is not fully understood, there is significant empirical and some theoretical evidence (e.g.,~\cite{gunasekar2017implicit}) that a similar minimum norm inductive bias is present.
Yet another type of inductive bias related to averaging is used in random forests.
Averaging potentially non-smooth interpolating trees leads to an interpolating solution with a higher degree of smoothness; this averaged solution performs better than any individual interpolating tree.

Remarkably, for kernel machines all three methods lead to the same minimum norm solution. Indeed, the minimum norm interpolating classifier, $\optninf$, can be obtained directly by explicit norm minimization (solving an explicit system of linear equations), through SGD or by averaging trajectories of Gaussian processes (computing the posterior mean~\cite{rasmussen2004gaussian}).  

\paragraph{Optimization and practical considerations.}
In our experiments, appropriately chosen ``modern''  models usually outperform the optimal ``classical'' model on the test set.
But another important practical advantage of over-parametrized models is in  optimization.
There is a growing understanding that larger models are ``easy'' to optimize as local methods, such as SGD, converge to global minima of the training risk in over-parametrized regimes (e.g.,~\cite{soltanolkotabi2018theoretical}).
Thus, large interpolating models can have low test risk and be easy to  optimize at the same time, in particular with SGD~\cite{ma2018power}.
It is likely that the models to the left of the interpolation peak have optimization properties qualitatively different from those to the right, a distinction of significant practical import.

\paragraph{Outlook.}
The classical U-shaped bias-variance  trade-off curve has shaped our view of model selection and directed applications of learning algorithms in practice. The understanding of model performance developed in this work delineates the limits of classical analyses and opens new lines of enquiry
to study and compare computational, statistical, and mathematical properties of the classical and modern regimes in  machine learning. We hope that this perspective, in turn, will help practitioners choose models and algorithms for optimal performance.

\section*{Acknowledgments}
We thank Peter Bickel for editing the PNAS submission, and the anonymous reviewers for their helpful feedback. 
Mikhail Belkin, Siyuan Ma and Soumik Mandal were supported by NSF RI-1815697.
Daniel Hsu was supported by NSF CCF-1740833 and Sloan Research Fellowship. We thank Nvidia for donating GPUs used for this research.

\bibliographystyle{plainnat}

\appendix

\section{Approximation theorem}
\label{app:approx}

Suppose the training data $(x_1,y_1), \dotsc, (x_n,y_n)$ are sampled independently by drawing $x_i$ uniformly from a compact domain in $\R^d$, and assigning the label $y_i = h^*(x_i)$ using a target function $h^* \in \rkhs$.
Let $h \in \rkhs$ be another hypothesis that interpolates the training data $(x_1,y_1),\dotsc,(x_n,y_n)$.
The following theorem bounds the error of $h$ in approximating $h^*$.
\begin{theorem}
  \label{thm:wendland}
  Fix any $h^* \in \rkhs$.
  Let $(x_1,y_1),\dotsc,(x_n,y_n)$ be independent and identically distributed random variables, where $x_i$ is drawn uniformly at random from a compact cube\footnote{Same argument can be used for more general domains and probability distributions.} $\Omega \subset \R^d$, and $y_i = h^*(x_i)$ for all $i$.
  There exists absolute constants $A, B > 0$ such that, for any interpolating $h \in \rkhs$ (i.e., $h(x_i) = y_i$ for all $i$), so that with high probability
  \[
    \sup_{x \in \Omega} |h(x) - h^*(x)| < Ae^{-B(n/\log n)^{1/d}}\left(\|h^*\|_\rkhs + \|h\|_\rkhs \right) .
  \]
\end{theorem}
\begin{proof}[Proof sketch]
Recall that {\it the fill} $\kappa_n$ of the set of points $x_1,\ldots,x_n$ in $\Omega$ is a measure of how well these points cover $\Omega$: $\kappa_n = \max_{x \in \Omega} \min_{x_j \in \{x_1,\ldots,x_n\}} \|x - x_j\|$. It is easy to verify (e.g., by taking an appropriate grid partition of the cube $\Omega$ and applying the union bound)  
that with high probability $\kappa_n = O(n / \log n)^{-1/d}$. 

Consider now a function $f(x) := h(x) - h^*(x)$. We observe that $f(x_i)=0$ and by the triangle inequality $\|f\|_\rkhs \le \|h^*\|_\rkhs + \|h\|_\rkhs$. Applying Theorem 11.22 in~\cite{wendland_2004} to $f$ yields the result. 
\end{proof}
The minimum norm interpolating function $\optninf$ has norm no larger than that of $h^*$ (by definition) and hence achieves the smallest bound in Theorem~1.
While these bounds apply only in the noiseless setting, they provide a justification for the inductive bias based on choosing a solution with a small norm.
Indeed, there is significant empirical evidence that minimum norm interpolating solutions generalize well on a variety of datasets, even in the presence of large amounts of label noise~\cite{belkin2018understand}.

\section{Experimental setup}

To demonstrate the double descent risk curve, we train a number of representative models including neural networks, kernel machines and ensemble methods on several widely used datasets that involve images, speech, and text.

\paragraph{Datasets.}

\begin{table}[b]
\centering
\caption{Descriptions of datasets. In experiments, we use subsets to reduce the computational cost. }
\label{tbl: dataset}
\begin{tabular}{|c||c|c|c|}
\hline
Dataset & \begin{tabular}[c]{@{}c@{}}Size of\\ full training set\end{tabular} & \begin{tabular}[c]{@{}c@{}}Feature\\ dimension ($d$)\end{tabular} & \begin{tabular}[c]{@{}c@{}}Number of\\ classes\end{tabular} \\ \hline\hline
CIFAR-10 & $5 \cdot 10^4$ & 1024 & 10 \\ \hline
MNIST & $6 \cdot 10^4$ & 784 & 10 \\ \hline
SVHN & $7.3 \cdot 10^4$ & 1024 & 10 \\ \hline
TIMIT & $1.1 \cdot 10^6$ & 440 & 48 \\ \hline
20-Newsgroups & $1.6 \cdot 10^4$ & 100 & 20 \\ \hline
\end{tabular}
\end{table}

Table~\ref{tbl: dataset} describes the datasets we use in our experiments.
These datasets are for classification problems with more than two classes, so we adopt the one-versus-rest strategy that maps a multi-class label to a binary label vector (one-hot encoding).
For the image datasets---namely MNIST~\cite{lecun1998gradient}, CIFAR-10~\cite{krizhevsky2009learning}, and SVHN~\cite{netzer2011reading}---color images are first transformed to grayscale images, and then the maximum range of each feature is scaled to the interval $[0, 1]$.
For the speech dataset TIMIT~\cite{garofolo1993darpa}, we normalize each feature by its z-score. 
For the text dataset 20-Newsgroups~\cite{Lang95}, we transform each sparse feature vector (bag of words) into a dense feature vector by summing up its corresponding word embeddings obtained from~\cite{pennington2014glove}.

For each dataset, we subsample a training set (of size $n$) uniformly at random without replacement.
For the 20-Newsgroups dataset, which does not have a test set provided, we randomly pick $1/8$ of the full dataset for use as a test set.

\paragraph{Model training.}

Each model is trained to minimize the squared loss on the given training set.
Without regularization, such model is able to interpolate the training set when its capacity surpasses certain threshold (interpolation threshold).
For comparison, we report the test/train risk for  zero-one and squared loss. 
In experiments for neural networks and ensemble methods, we repeat the same experiment five times and report the mean for the risks. RFF and Random ReLU experimental results were reported based on a single run as the results were empirically highly consistent.

\begin{figure}
  \centering
    \includegraphics[width=.9\textwidth]{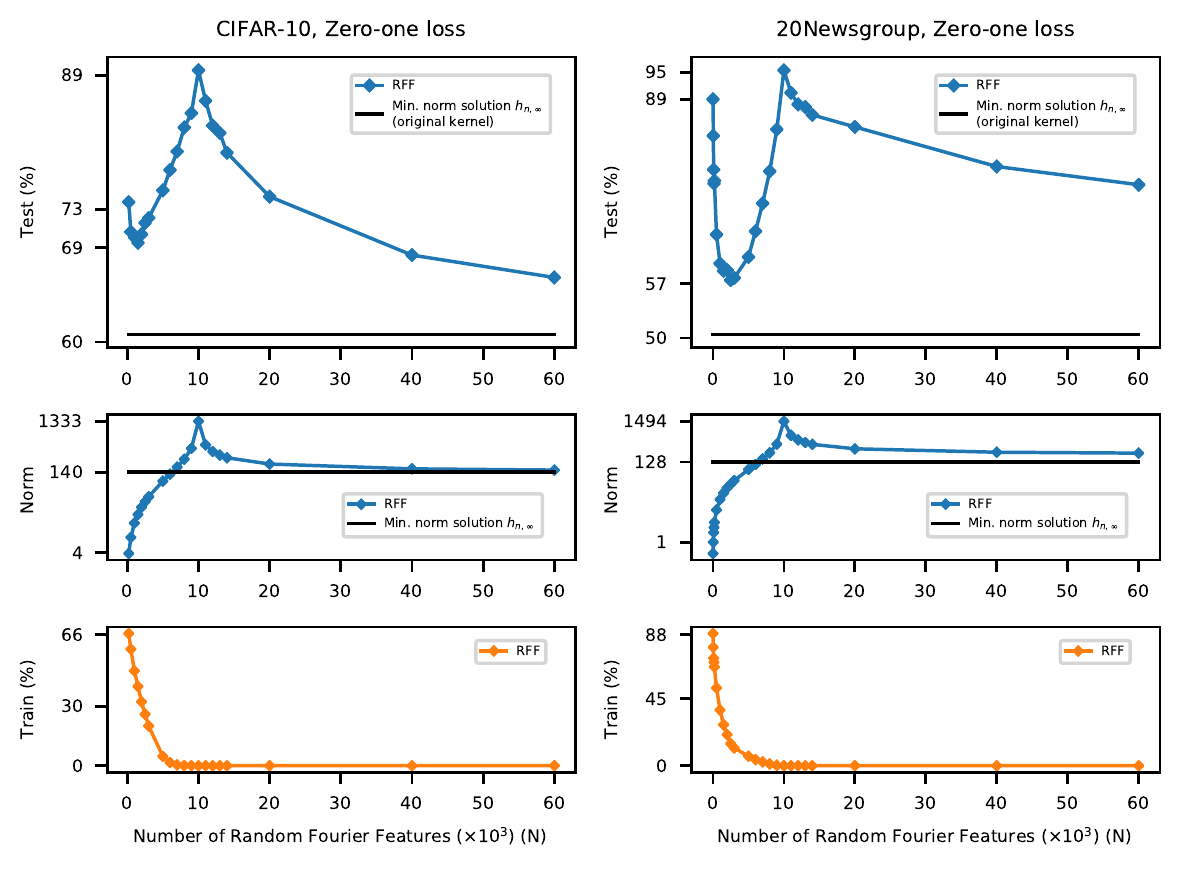}    
  \caption{{\bf Double descent risk curve for RFF model.} Test risks (log scale), coefficient $\ell_2$ norms (log scale), and training risks of the RFF model predictors $\optnN$ learned on subsets of CIFAR-10 and 20Newsgroups ($n=10^4$). The interpolation threshold is achieved at $N=10^4$. }
  \label{fig:RFF_CIFAR-10_20NEWS}
\end{figure}

\begin{figure}
  \centering
    \includegraphics[width=.9\textwidth]{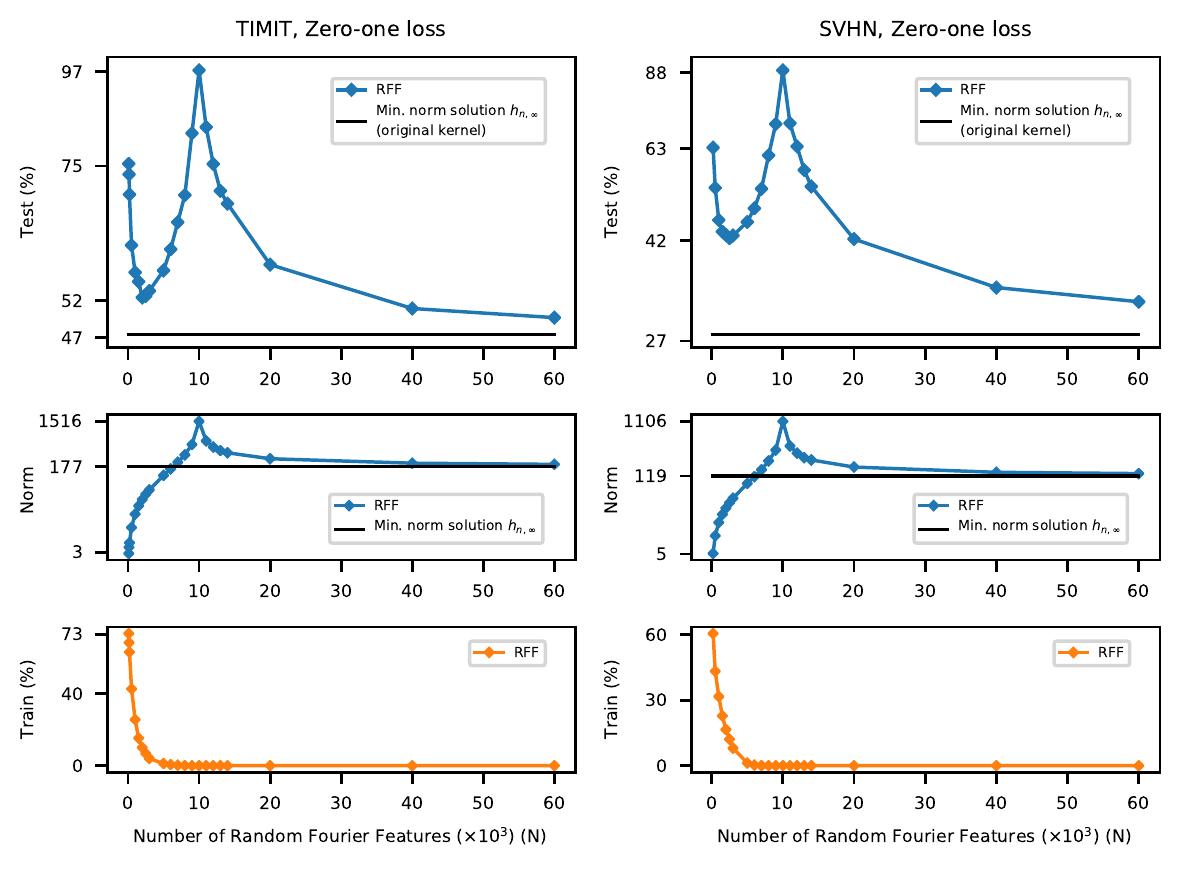}
  \caption{{\bf Double descent risk curve for RFF model.} Test risks (log scale), coefficient $\ell_2$ norms (log scale), and training risks of RFF model predictors $\optnN$ learned on subsets of TIMIT and SVHN ($n=10^4$). The interpolation threshold is achieved at $N=10^4$.}
  \label{fig:RFF_TIMIT_SVHN}
\end{figure}

\section{Additional experimental results for neural networks}
\label{app:RFF+NN}

\subsection{Random Fourier Feature models}\label{app:RFF}

We provide additional experimental results  for several real-world datasets. Figure~\ref{fig:RFF_CIFAR-10_20NEWS} illustrates double descent behavior for CIFAR-10 and 20Newsgroup. Figure~\ref{fig:RFF_TIMIT_SVHN} shows similar curves of zero-one loss for TIMIT and SVHN. 
The random feature vectors $v_1,\dotsc,v_N$ are sampled independently from $\mathcal{N}(0,\sigma^{-2} \cdot I)$, the mean-zero normal distribution in $\R^d$ with covariance $\sigma^{-2} \cdot I$.
The bandwidth parameter $\sigma$ is set to 5, 5, 5, 0.1, and 16 for MNIST, SVHN, CIFAR-10, 20-Newsgroup, and TIMIT, respectively.

\subsection{Random ReLU Feature models}\label{app:ReLU}

\begin{figure}
  \centering
    \includegraphics[width=.9\textwidth]{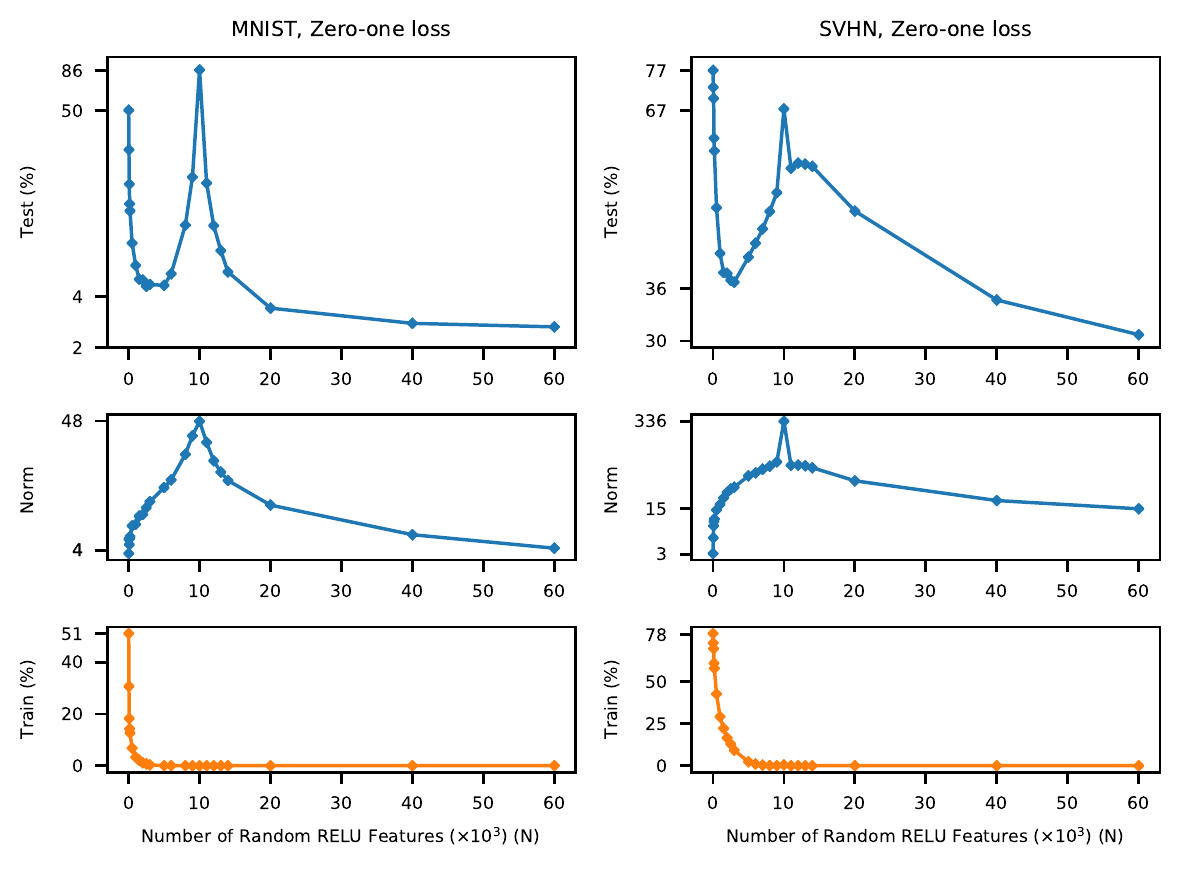}
  \caption{{\bf Double descent risk curve for Random ReLU model.} Test risks (log scale), coefficient $\ell_2$ norms (log scale), and training risks of the Random ReLU Features model predictors $\optnN$ learned on subsets of MNIST and SVHN data ($n=10^4$). The interpolation threshold is achieved at $N=10^4$. Regularization of $4\cdot10^{-6}$ is added for SVHN to ensure numerical stability near interpolation threshold.}
  \label{fig:RELU_MNIST_SVHN}
\end{figure}

We show that the double descent risk curve also appears with Random ReLU feature networks~\cite{cho2009ReLU}.
Such networks are similar to the RFF models, except that they use the ReLU transfer function.
Specifically, the Random ReLU features model family $\cHN$ with $N$ parameters consists of functions $h \colon \R^d \to \R$ of the form
\[
  h(x) = \sum_{k=1}^N a_k \phi(x; v_k)
  \quad\text{where}\quad
  \phi(x; v) :=\max( \langle v,x \rangle,0) .
\]
The vectors $v_1,\dotsc,v_N$ are sampled independently from uniform distribution over surface of unit sphere in $\R^d$.
The coefficients $a_k$ are learned using linear regression.
Figure~\ref{fig:RELU_MNIST_SVHN} illustrates zero-one loss with Random ReLU features for MNIST and SVHN data.
Ridge regularization with parameter $\lambda := 4\cdot 10^{-6}$ is added in SVHN experiments to ensure numerical stability near the interpolation threshold.
For MNIST experiments, no regularization is added.
We observe that the resulting risk curves and the norm curves are very similar to those for RFF.

\subsection{Fully connected neural networks}\label{app:FNN}

In our experiments, we use fully connected neural networks with a single hidden layer.
To control the capacity of function class, we vary the number of hidden units.
We use stochastic gradient descent (SGD) to solve the ERM optimization problem in this setting.

The ERM optimization problem in this setting is generally more difficult than that for RFF and ReLU feature models due to a lack of analytical solutions and non-convexity of the problem.
Consequently, SGD is known to be sensitive to initialization.
To mitigate this sensitivity, we use a ``weight reuse'' scheme with SGD in the under-parametrized regime ($N < n$), where the parameters obtained from training a smaller neural network are used as initialization for training larger networks.
This procedure, detailed below, ensures decreasing training risk as the number of parameters increases. 
In the over-parametrized regime ($N \geq n$), we use standard (random) initialization, as typically there is no difficulty in obtaining near-zero training risk. 

Additional experimental results for neural networks are shown in Figure~\ref{fig:fcnn-s}. Results for MNIST and CIFAR-10 with weight reuse are reported in Figure~\ref{fig:fcnn-s}(a) and Figure~\ref{fig:fcnn-s}(b). Results for MNIST without weight reuse are reported in Figure~\ref{fig:fcnn-s}(c). In this setting, all models are randomly initialized. While the variance is significantly larger, and the training loss is not monotonically decreasing, the double descent behavior is still clearly discernible.

We now provide specific details below. We use SGD with standard momentum (parameter value $0.95$) implemented in~\cite{chollet2015keras} for training.
In the weight reuse scheme, we assume that we have already trained a smaller network with $H_1$ hidden units. To train a larger network with $H_2>H_1$ hidden units, we initialize the first $H_1$ hidden units of the larger network to the weights learned in the smaller network. The remaining weights are initialized with normally distributed random numbers (mean $0$ and variance $0.01$). The smallest network is initialized using standard Glorot-uniform distribution~\cite{glorot2010understanding}. For networks smaller than the interpolation threshold, we decay the step size by $10\%$ after each of $500$ epochs, where an epoch denotes a pass through the training data.
For these networks, training is stopped after classification error reached zero or $6000$ epochs, whichever happens earlier. For networks larger than interpolation threshold, fixed step size is used, and training is stopped after $6000$ epochs.

\begin{figure}
  \begin{tabular}{ccc}
  \includegraphics[width=0.31\textwidth]{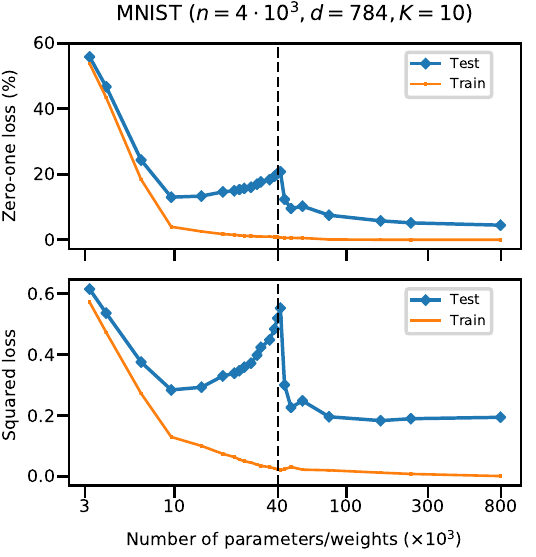} &
  \includegraphics[width=0.31\textwidth]{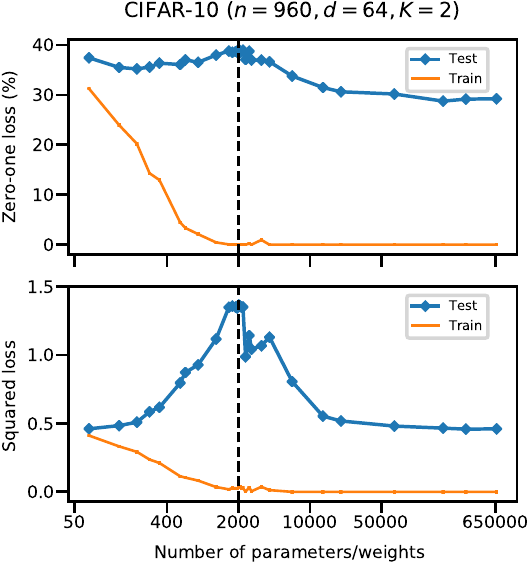} &
  \includegraphics[width=0.31\textwidth]{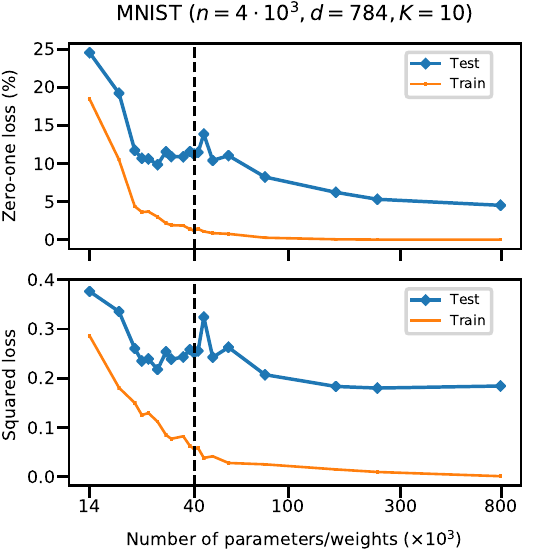} \\
  {\bf (a)} & {\bf (b)} & {\bf (c)}
  \end{tabular}
  \caption{{\bf Double descent risk curve for fully connected neural networks.}  
  In each plot, we use a dataset with $n$ subsamples of $d$ dimension and $K$ classes for training.
  We use networks with a single hidden layer. 
  For network with $H$ hidden units, its number of parameters is $(d+1) \cdot H + (H+1) \cdot K$. 
  The interpolation threshold is observed at $n \cdot K$ and is marked by black dotted line in figures.
  ({\bf a}) Weight reuse before interpolation threshold and random initialization after it on MNIST.
  ({\bf b}) Same, on a subset of CIFAR-10 with 2 classes (cat, dog) and downsampled image features ($8 \times 8$).
  ({\bf c}) No weight reuse (random initialization for all ranges of parameters).}
  \label{fig:fcnn-s}
\end{figure}

\subsection{Synthetic model} \label{app:RFF1D}

We now discuss the nature of the double descent risk curve in the context of a simple synthetic model, which can be viewed as a version of RFF for  functions on the one-dimensional circle.
Consider the class $\cH$ of periodic complex-valued
functions on the interval $[0,2\pi]$, and let \[ e_k(x) := \exp(\sqrt{-1}(k{-}1)x) \] for positive integers $k$.
Fix a probability distribution $p=(p_1,p_2,\ldots)$ on the positive integers.
For each integer $N$, we generate a random function class $\cHN$ by (i) sampling independently from $p$ until $N$ distinct indices $k_1,\dotsc,k_N$ are chosen, and then (ii) let $\cHN$ be the linear span of $e_{k_1},\dotsc,e_{k_N}$.
Here, $N$ is the number of parameters to specify a function in $\cHN$ and also reflects the capacity of $\cHN$.

We generate data from the following model:
\[ y_i = h^*(x_i) + \varepsilon_i \]
where the target function $h^* = \sum_k \alpha_k^* e_k$ is in the span of the $e_k$, and $\varepsilon_1,\dotsc,\varepsilon_n$ are independent zero-mean normal random variables with variance $\sigma^2$.
The $x_1,\dotsc,x_n$ themselves are drawn uniformly at random from $\{2\pi j/M : j = 0,\dotsc,M-1\}$ for $M := 4096$.
We also let $\alpha_k^* := p_k$ for all $k$, with $p_k \propto 1/k^2$.
The signal-to-noise ratio (SNR) is $\E[ h^*(x_i)^2 ] / \sigma^2$.

Given data $(x_1,y_1),\dotsc,(x_n,y_n) \in [0,2\pi] \times \R$, we learn a function from the function class $\cHN$ using empirical risk minimization, which is equivalent to ordinary least squares over an $N$-dimensional space.
Interpolation is achieved when $N \geq n$, so in this regime, we choose the interpolating function $h = \sum_{j=1}^N \alpha_{k_j} e_{k_j}$ of smallest (squared) norm $\|h\|_{\cH}^2 = \sum_k \alpha_k^2/p_k$.

Our simulations were carried out for a variety of sample sizes ($n \in \{2^6,2^7,\dotsc,2^{11}\}$) and are all repeated independently $20$ times; our plots show averages over the $20$ trials.
The results confirm our hypothesized double descent risk curve, as shown in Figure~\ref{fig:synthetic} for $n=256$; the results are similar for other $n$.
The peak occurs at $N = n$, and the right endpoint of the curve is lower than the bottom of the U curve.
The norm of the learned function also peaks at $N = n$ and decreases for $N > n$.

\begin{figure}
  \centering
  \begin{tabular}{cc}
    \includegraphics[width=.45\textwidth]{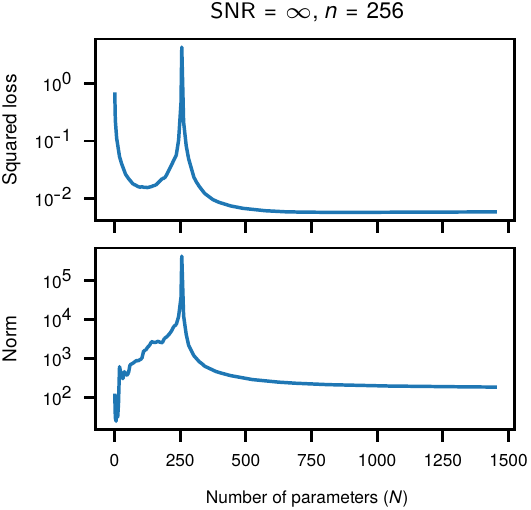} &
    \includegraphics[width=.45\textwidth]{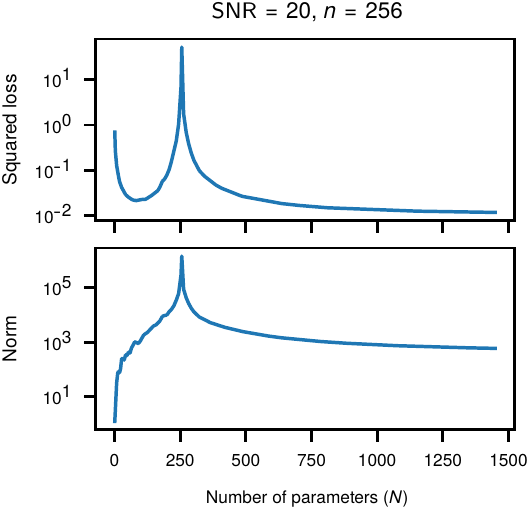}
  \end{tabular}
  \caption{{\bf Results from the synthetic model at SNR = $\infty$ and SNR = $20$.}
  Top: excess test risk under squared loss of learned function.
  Bottom: norm of learned function $\|h\|_{\rkhs}$. For $n = 256$ training samples, the interpolation threshold is reached at $N = 256$.
  }
  \label{fig:synthetic}
\end{figure}

\section{Additional results with Random Forests}
\label{app:forests}

\begin{figure}
  \centering
  \begin{tabular}[t]{ccc}
    \includegraphics[width=0.31\textwidth]{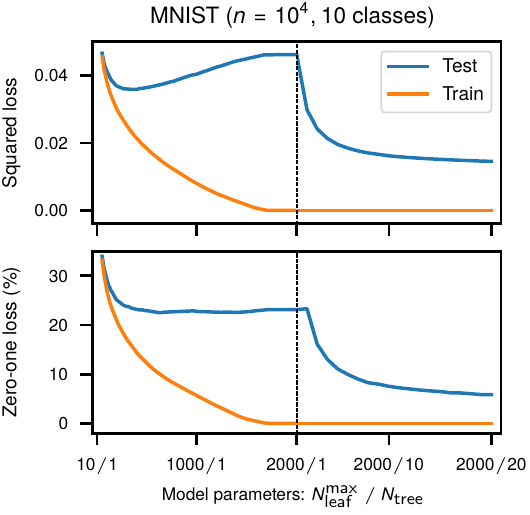} &
    \includegraphics[width=0.31\textwidth]{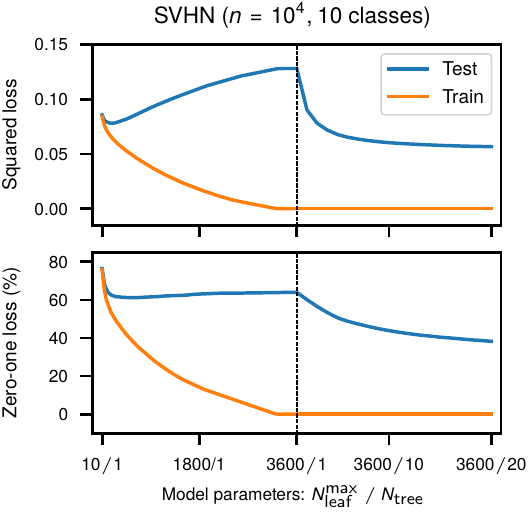} &
    \includegraphics[width=0.31\textwidth]{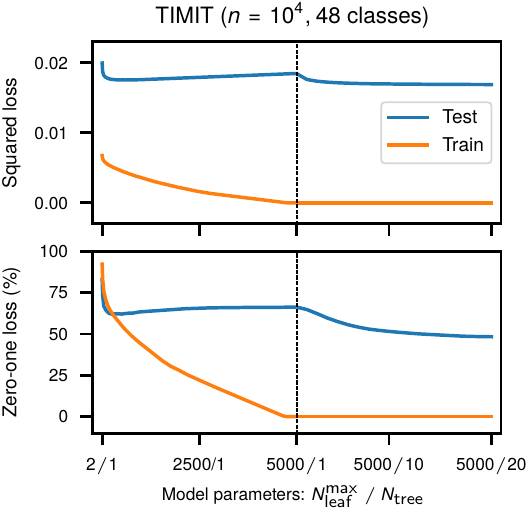}
    \\
    \multicolumn{3}{c}{\bf (a)} \\
    \includegraphics[width=0.31\textwidth]{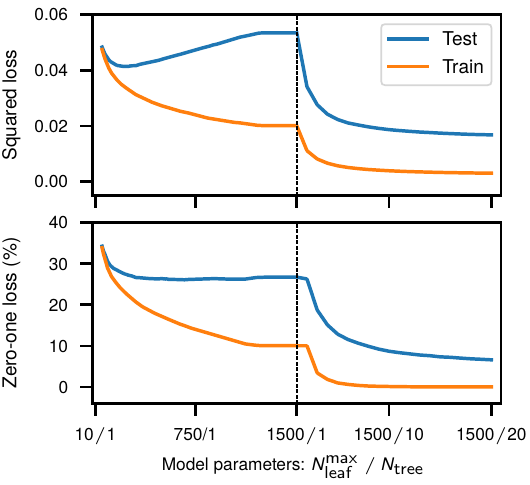} &
    \includegraphics[width=0.31\textwidth]{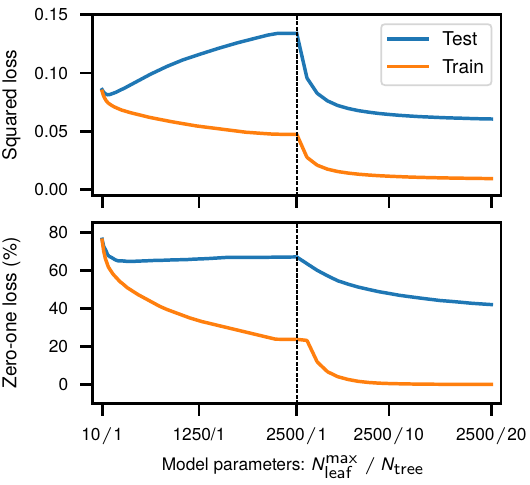} &
    \includegraphics[width=0.31\textwidth]{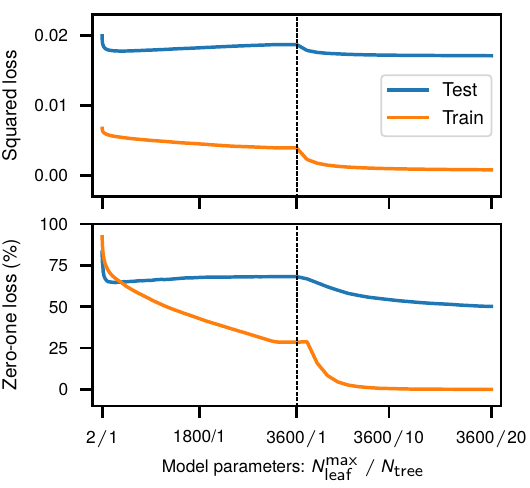}
    \\
    \multicolumn{3}{c}{\bf (b)}
  \end{tabular}
\caption{{\bf Double descent risk curve for random forests.}
In all plots, the double descent risk curve is observed for random forest with increasing model complexity on regression tasks. Its complexity is controlled by the number of trees $N_\mathsf{tree}$ and the maximum number of leaves allowed for each tree $N_\mathsf{leaf}^\mathsf{max}$.
({\bf a}) Without bootstrap re-sampling, a single tree can interpolate the training data.
({\bf b}) With bootstrap re-sampling, multiple trees are needed to interpolate the data.}\label{fig:random-forest-s}
\end{figure}

We train standard random forests introduced by Breiman~\cite{breiman2001random} for regression problems.
When splitting a node, we randomly select a subset of features whose number is the square root of the number of the total features, a setting which is widely used in mainstream implementations of random forest. We control the capacity of the model class by choosing the number of trees ($N_\mathrm{tree}$) and limiting the maximum number of leaves in each tree ($N_\mathrm{leaf}^\mathrm{max}$). We put minimum constraints on the growth of each tree: there is no limit for the tree depth and we split each tree node whenever it is possible.

To interpolate the training data, we disable the bootstrap re-sampling for results in Figure~\ref{fig:random-forest-s}(a), which has been investigated under the name ``Perfect random tree ensembles'' by Cutler et al.~\cite{cutler2001pert}. We see clear double decent risk curve (with both squared loss and zero-one loss) as we increase the capacity of the model class (although the U-shaped curve is less apparent with zero-one loss). In Figure~\ref{fig:random-forest-s}(b), we run the same experiments with bootstrap re-sampling enabled, which show similar double decent risk curves.

\section{Results with $L_2$-boosting}
\label{app:boosting}

\begin{figure}
\centering
  \begin{tabular}{cc}
    \includegraphics[width=0.45\textwidth]{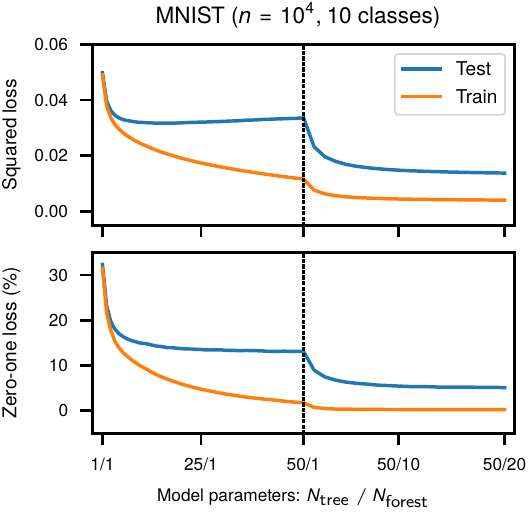} &
    \includegraphics[width=0.45\textwidth]{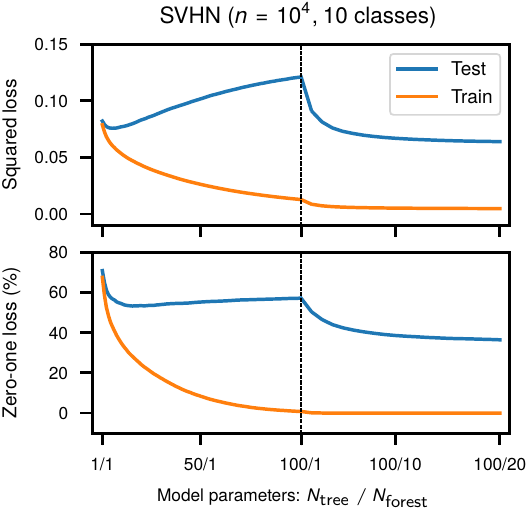}
  \end{tabular}
  \caption[]{{\bf Double descent risk curve for $L_2$-boosting trees.}
In both plots, we increase the model complexity by first increasing the number of boosting (random) trees ($N_\mathsf{tree}$) which form a forest, then averaging several such forests ($N_\mathsf{forest}$). Each tree is constrained to have no more than 10 leaves. For fast interpolation, the gradient boosting is applied with low shrinkage parameter ($0.85$).
}
  \label{fig:l2-boosting}
\end{figure}

We now show double descent risk curve for $L_2$-boosting (random) trees introduced by Friedman~\cite{friedman2001greedy}.
When splitting a node in a tree, we randomly select a subset of features whose number is the square root of the number of the total features.
We constrain each tree to have a small number of leaves (no more than 10).
As the number of trees increases, the boosted trees gradually interpolate the training data and form a forest.
To quickly reach interpolation, we adopt low shrinkage (parameter value $0.85$) for gradient boosting.
To go beyond the interpolation threshold, we average the predictions of several such forests which are randomly constructed and trained with exactly same hyper-parameters. The capacity of our model is hence controlled by the number of forests ($N_\mathrm{forest}$) and the number of trees ($N_\mathrm{tree}$) in each forest.

Figure~\ref{fig:l2-boosting} shows the change of train and test risk as the model capacity increases.
We see the double descent risk curve for both squared loss and zero-one loss.
We also observe strong over-fitting under squared loss before the interpolation threshold.
For similar experiments with high shrinkage (parameter value $0.1$), the double descent risk curve becomes less apparent due to the regularization effect of high shrinkage~\cite{friedman2001elements}.

\end{document}